\documentclass{article}

\usepackage{microtype}
\usepackage{graphicx}
\usepackage{subcaption}
\usepackage{booktabs} 
\usepackage{wrapfig}

\usepackage{soul}
\usepackage{makecell}
\usepackage{amsthm}

\newtheoremstyle{notestyle}
  {3pt}
  {3pt}
  {\itshape}
  {}
  {\bfseries}
  {.}
  {.5em}
  {\thmname{#1}\thmnumber{ #2}\thmnote{ (#3)}}
\theoremstyle{notestyle}
\newtheorem*{note}{Note}

\usepackage{hyperref}

\usepackage[accepted]{icml2024}

\usepackage{amsmath}
\usepackage{amssymb}
\usepackage{mathtools}
\usepackage{amsthm}

\usepackage[capitalize,noabbrev]{cleveref}

\theoremstyle{plain}
\newtheorem{theorem}{Theorem}[section]

\theoremstyle{definition}

\theoremstyle{remark}

\usepackage{algorithm}
\usepackage{algpseudocode}
\usepackage{pifont}

\icmltitlerunning{Non-Vacuous Generalization Bounds for Large Language Models}

\begin{document}

\twocolumn[
\icmltitle{Non-Vacuous Generalization Bounds for Large Language Models}

\icmlsetsymbol{equal}{*}

\begin{icmlauthorlist}
\icmlauthor{Sanae Lotfi}{equal,nyu}
\icmlauthor{Marc Finzi}{equal,cmu}
\icmlauthor{Yilun Kuang}{equal,nyu}
\\
\vspace{0.5em}
\icmlauthor{Tim G. J. Rudner}{nyu}
\icmlauthor{Micah Goldblum}{nyu}
\icmlauthor{Andrew Gordon Wilson}{nyu}
\end{icmlauthorlist}

\icmlaffiliation{nyu}{New York University}
\icmlaffiliation{cmu}{Carnegie Mellon University}

\icmlcorrespondingauthor{Sanae Lotfi}{sl8160@nyu.edu}
\icmlcorrespondingauthor{Marc Finzi}{maf820@nyu.edu} 
\icmlcorrespondingauthor{Andrew Gordon Wilson}{\\andrewgw@cims.nyu.edu}

\icmlkeywords{Machine Learning, ICML}

\vskip 0.3in
]

\printAffiliationsAndNotice{\icmlEqualContribution}

\begin{abstract}
Modern language models can contain billions of parameters, raising the question of whether they can generalize beyond the training data or simply parrot their training corpora. We provide the first non-vacuous generalization bounds for pretrained large language models (LLMs), indicating that language models are capable of discovering regularities that generalize to unseen data. In particular, we derive a compression bound that is valid for the unbounded log-likelihood loss using prediction smoothing, and we extend the bound to handle subsampling, accelerating bound computation by orders of magnitude on massive datasets. To achieve the extreme level of compression required for non-vacuous bounds, we devise SubLoRA, a simple low-dimensional nonlinear parameterization that leads to non-vacuous generalization bounds for models with nearly a billion parameters. Finally, we use our bounds to understand LLM generalization and find that larger models have better generalization bounds and are more compressible than smaller models.
\end{abstract}

\section{Introduction}

Do large language models (LLMs) merely memorize the training data, and if so, are they able to meaningfully generalize beyond their training set?
This question is central to understanding LLMs as they continue to grow in capacity and are capable of memorizing and parroting training examples verbatim~\citep{brown2020language,palm, carlini2020extracting,carlini2023quantifyingiclr}.

In this work, we address the question of generalization in LLMs by computing the first non-vacuous generalization bounds for language model pretraining on next token prediction, thereby providing a mathematical guarantee that LLMs are able to generalize beyond their training data.

Although significant progress has been made in constructing non-vacuous generalization bounds for image classification models using the PAC-Bayes framework \citep{catoni2007pac} in conjunction with extreme levels of model compression \citep{zhou2019non, lotfi2022pac}, non-vacuous generalization bounds for large language models remain elusive.

Compared to image classification models, constructing non-trivial bounds for language models presents additional challenges:
(i) LLMs are trained on autoregressive token prediction, and thus token predictions are not independent; (ii) the relevant negative log-likelihood (NLL) metric (bits per dimension) is a continuous and unbounded random variable for which previously used non-vacuous PAC-Bayes bounds are invalid; and (iii) LLMs have orders of magnitude more parameters than image classification models.
To address these challenges, we derive new generalization bounds that can be applied to the unbounded bits per dimension objective.
We also introduce an extension of these bounds which can be computed using only a subset of the training data, making bound computation $900$ times faster on the OpenWebText dataset, which has more than $9$ billion tokens.  

Achieving the extreme level of compression required to obtain non-vacuous generalization bounds for LLMs is another challenge.
To this end, we devise SubLoRA (Subspace-Enhanced Low-Rank Adaptation):
simple nonlinear parameterization for LLMs that makes it possible to smoothly vary the level of compression while maintaining expressivity. SubLoRA combines low-rank adaptation (LoRA) \citep{hu2021lora}, originally proposed for efficient \emph{fine-tuning}, with subspace training \citep{li2018measuring,lotfi2022pac} to \emph{pretrain} highly compressed LLMs from scratch.

\begin{figure*}[t!]
    \hspace{10pt}
    \begin{subfigure}{0.6\textwidth} 
        \centering
        \includegraphics[width=\linewidth]{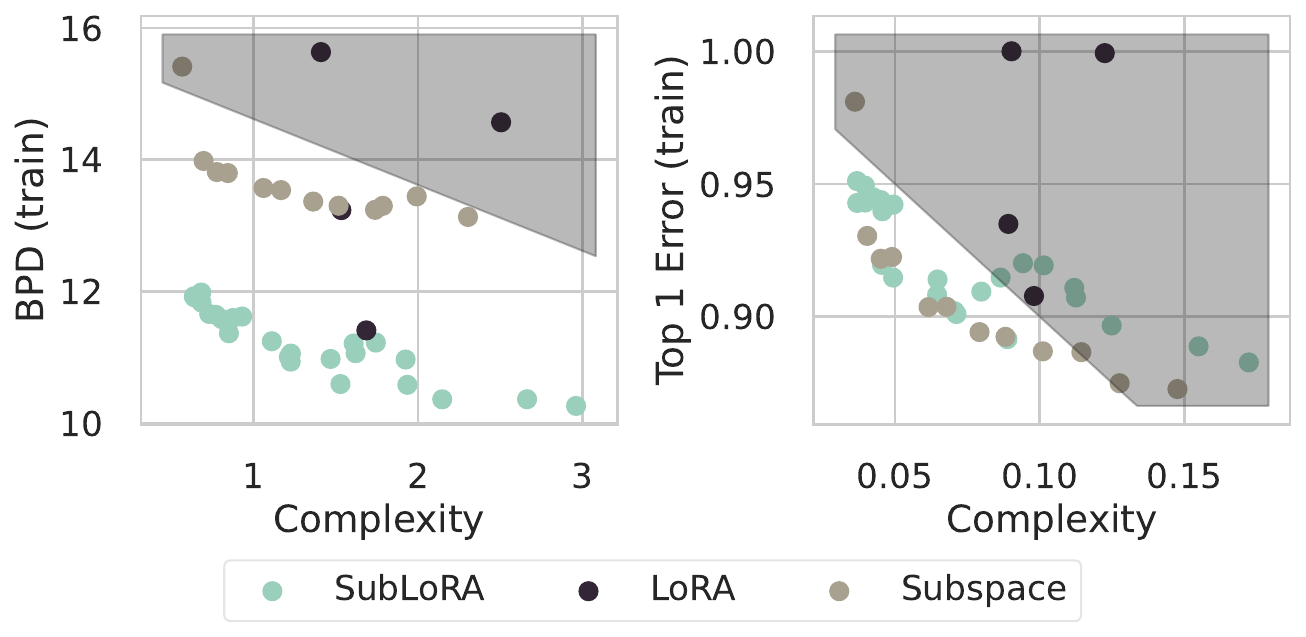}
        \label{fig:subfig1}
    \end{subfigure}
    \hspace{20pt}
    \begin{subfigure}{0.27\textwidth} 
        \centering
        \includegraphics[width=\linewidth]{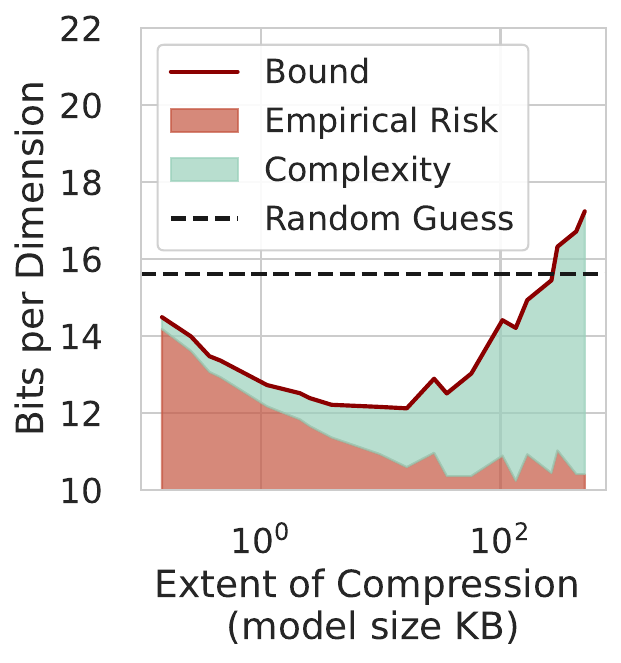} 
        \label{fig:subfig2}
    \end{subfigure}
    \caption{\textbf{Finding solutions that simultaneously achieve low training error and low complexity with SubLoRA.} 
    \textbf{(Left):} The Pareto frontier of model complexity (the 2nd term in \autoref{eq:bound}) and the empirical risk (bits per dimension (BPD) and Top-1 Error) of language models using LoRA and subspace compression for next token prediction pretraining. The generalization bound is formed from the sum of the two axes (lower is better), with the shaded region showing where bounds are vacuous. 
    Combining both LoRA and subspace compression in the form of SubLoRA yields the best bounds, while using LoRA alone yields vacuous bounds for top-1 error. 
    \textbf{(Right):} SubLoRA enables a smooth tradeoff over the extent of model compression for a fixed model, finding the degree of compression that is optimal for the situation in constructing the generalization bounds. We plot the contributions of the empirical risk and the complexity term to the bound as a function of this degree of compression.}
    \label{fig:fig1}
\vspace{-0.5em}
\end{figure*}

Combining the above-described theoretical and practical contributions, we achieve \textit{the first non-vacuous bounds for large language models}.
To highlight the efficiency of our new compression technique, we compare SubLoRA to LoRA and subspace training in \Cref{fig:fig1} (left).
We compute two metrics that we define as follows: Top-1 Error, which is the $0$-$1$ error in predicting the next token averaged over a given document; and the bits per dimension metric, which corresponds to the average negative log-likelihood per document.
The shaded region highlights where bounds become vacuous, with SubLoRA achieving non-vacuous bounds for both bits per dimension and Top-1 Error.
The term \textit{vacuous} refers to the random guess performance which is $\log_2 V$ for BPD and $1-1/V$ for Top-1 Error, where $V$ is the vocabulary size.
In contrast, we see that only using LoRA achieves vacuous bounds for Top-1 Error and only using subspace achieves a high value of empirical BPD. Despite the simplicity of SubLoRA, it has an improved ability to trade-off model complexity with training error.
In \Cref{fig:fig1} (right), we highlight the trade-off between model complexity and empirical risk in the generalization bounds as we vary the level of compression.

We summarize our contributions as follows:\vspace*{-7pt}
\begin{itemize}
\item \textbf{Novel bounds for the unbounded negative log-likelihood objective:} we introduce novel bounds specifically tailored to account for the unbounded continuous bits-per-dimension loss, commonly used to evaluate LLMs for next-token prediction.
\item \textbf{Subsampling bounds for practical bound evaluation:} To make the evaluation of the bounds practical on LLMs with massive datasets, we derive subsampling-based bounds that allow for efficient evaluation.
In practice, the evaluation of the bound takes 45 minutes on a single GPU instead of 3 days on 8 GPUs in parallel for the OpenWebText dataset.
\item \textbf{A simple yet powerful nonlinear subspace compression for LLMs:} as we show in \Cref{fig:fig1}, using LoRA alone to compress the discrepancy between the random initialization and a learned model leads to vacuous bounds for the top-1 error. 
At the same time, linear subspace training alone does not unlock the full compression potential of LLMs compared to a nonlinear compression scheme. 
We show that a combination of these two approaches, while simple, yields a strong nonlinear compression of the model, which leads to the best generalization bounds for LLMs.
\item \textbf{Non-vacuous generalization bounds for models with nearly a billion parameters:} our work not only introduces the first non-vacuous generalization bounds for LLMs, but it also extends these bounds to models with over $800$ million parameters, demonstrating the scalability of our compression technique. 
\item \textbf{Improved understanding of generalization in LLMs:} as we increase the size of models, we find that they are able to find more compressed representations of the data and achieve better bounds, therefore disproving the claim that larger LLMs are simply better at regurgitating their training data. 
\end{itemize}

These contributions collectively offer mathematical proof that large language models are, in fact, powerful knowledge compressors and are capable of generalization beyond their training samples, especially as their scale increases. 
To the best of our knowledge, our work is the first to show that generalization bounds improve with more parameters on models of practical sizes, in line with the empirical benefits of large models.
We make our code \href{https://github.com/Sanaelotfi/SubLoRA-bounds-for-LLMs}{\underline{available here}}.

\section{Related Work}

\textbf{Generalization bounds.} Neural networks have seen widespread adoption because of their strong performance on new unseen test samples, known as \emph{generalization}.  Early generalization theory literature bounded the difference in training and test error, called the \emph{generalization gap}, using complexity measures like VC-dimension \citep{vapnik1991principles} and Rademacher complexity \citep{bartlett2002rademacher}.
These generalization bounds were vacuous for neural networks, which are often flexible enough to fit randomly labeled training data \citep{zhang2021understanding}.  
The flexibility of neural networks and its negative impact on these classical bounds calls into question why they generalize. Neural networks are so flexible that they have parameter vectors where they fit their training data and simultaneously assign incorrect labels to testing data, and they also have parameter vectors where they fit their training data and instead assign correct labels to the testing data. Why do such flexible models actually make correct test predictions in practice?

PAC-Bayes generalization theory bridges this gap by leveraging the fact that while neural networks are highly flexible and can fit random labels, they encode a preference for the correct ones \citep{catoni2007pac, dziugaite2017computing}.  Unlike earlier generalization bounds which measured complexity merely as a function of the hypothesis class, PAC-Bayes generalization bounds reward models which have a strong prior that places its mass on parameter vectors that align with observed data.  This formulation allows one to draw a parallel between generalization and compressibility \citep{zhou2019non, lotfi2022pac}.  By placing disproportionate prior mass on compressible parameter vectors, achieving a tight bound simply requires finding a family of models (posterior) that well fit the training data.  Such compression bounds achieve the tightest guarantees to date on modern convolutional architectures and large-scale datasets, showcasing the strong inductive bias of neural networks and indicating that they can significantly compress their training sets \citep{lotfi2022pac}.
While PAC-Bayes has proven a very fruitful framework for devising such bounds, the insight on using a prior to bound the complexity of a given model does not require a posterior and can actually be incorporated into simpler finite hypothesis bounds that apply to deterministically trained models.

Recent generalization theory literature has expanded analysis to several relevant models---autoregressive time-series models and simple n-gram language models \citep{mcdonald2011generalization,bharadwaj2014pac,vankadara2022causal}.
In contrast, we construct bounds for autoregressive transformer-based language models.

\textbf{Existing bounds for unbounded objectives.} A number of works have explored techniques for generating generalization bounds on unbounded objective functions more generally, but these approaches are not practical for application to LLMs. A well established strategy relevant for e.g. linear regression with Gaussian errors is to bound the tails of the objective as subgaussian random variables, and then generalization bounds can be constructed for subgaussians more generally \citep{alquier2016properties,germain2016pac}. Other kinds of known tail behavior have also been exploited \citep{holland2019pac,kuzborskij2019efron}. For the NLL of a language model, there is no clear analogous tail behavior, so we must take a different approach.

\citet{haddouche2021pac} devise an approach for general unbounded objectives by constructing a hypothesis dependent bound on the objective, even if the objective is unbounded more generally. If the risk can be bounded $\sup_{x} R(h,x) \le Q(h)$ for a function $Q(h)$, then PAC-Bayes bounds can be constructed using $Q(h)$ even if $\sup_h Q(h) = \infty$. However, even though $Q(h)$ is finite for LLMs as there are only a finite number of inputs, $Q$ grows exponentially for NLL with the number of layers in the network and is closely related with the Lipschitz constant. For large models like LLMs, this value is far too large to be useful in constructing bounds.

\textbf{Language models and compression.}
Large language models are parameterized with as many as billions of parameters and, as a result, have a significant memory footprint, which makes pretraining, finetuning, and even evaluation challenging without access to large-scale computing infrastructure.
To reduce the memory footprint of large language models, a wide array of compression schemes has been proposed to enable evaluation, fine-tuning, and pre-training with limited computational resources.
Low-Rank Adaptation \citep[LoRA]{hu2021lora} freezes the pre-trained model weights and inserts trainable rank decomposition matrices into each attention layer of the transformer architecture used in large language models. Doing so allows for significantly reducing the number of trainable parameters for fine-tuning on downstream tasks.
For example, LoRA can reduce the number of trainable parameters in GPT-3 175B fine-tuned with Adam by a factor of 10,000 and the GPU memory requirement by a factor of 3.
Building on LoRA, Q-LoRA \citep{dettmers2023qlora} quantizes a pretrained model to 4-bits, adds a small set of learnable weights parameterized using LoRA, and then tunes these weights by backpropagating gradients through the quantized model.
Other compression methods for large language models use distillation \citep{liu2023llmqat}, sub-4-bit integer quantization \citep{kim2023memory, park2022lutgemm}, sparse quantized representations that identify and isolate outlier weights \citep{dettmers2023spqr}, weight quantization based on approximate second-order information \citep{frantal2022gptq}, or tensor-train decompositions \citep{xu2023tensorgpt}.

Achieving a good generalization bound has distinct requirements from the existing compression literature.
Unlike existing compression schemes for language models, which aim to accelerate inference and training or to reduce the memory footprint, we focus on specifying the trained model parameters in only few bits, even if doing so decreases neither latency nor memory requirements.

\section{Background}

\textbf{Subspace training.}
\citet{lotfi2022pac} train a compressible model by parameterizing a carefully constructed low-dimensional random subspace.
The weights $\theta \in \mathbb{R}^D$ are then defined as the sum of a random initialization $\theta_0$ and a projection $P \in \mathbb{R}^{D\times d}$ from a lower-dimensional subspace $w \in \mathbb{R}^d$: $\theta = \theta_0+Pw$.
$P$ is constructed as the Kronecker product of random Gaussian matrices \smash{$P = (Q_1 \otimes Q_2) / \sqrt{D}$ for $Q_1,Q_2 \sim \mathcal{N}(0,1)^{\sqrt{D} \times \sqrt{d}}$}, normalized so that \mbox{$P^\top P \approx I$}.
The weights $w$ can then be optimized over by backpropagating through the transformation.
With a learned quantization strategy---optimizing over quantized weights and the quantization levels---\citet{lotfi2022pac} use arithmetic coding to encode the weights using the empirical probabilities over quantization bins.

\textbf{Low Rank Adaptation (LoRA).} Similarly inspired by evidence that overparametrized models have low intrinsic dimensionality \citep{li2018measuring,aghajanyan2020intrinsic}, \citet{hu2021lora} propose LoRA as a parameter-efficient finetuning method.
Given a pretrained weight matrix $W_{\text{pretrained}}\in\mathbb{R}^{a\times b}$, LoRA decomposes its total update $\Delta W$ accumulated throughout finetuning as a product of two trainable low-rank matrices $U\in\mathbb{R}^{a\times r}, V\in\mathbb{R}^{r\times b}$ for $r \ll \min(a,b)$  while freezing $W_{\text{pretrained}}$.
Thus \smash{$W_{\text{finetuned}}=W_{\text{pretrained}}+\Delta W=W_{\text{pretrained}}+UV$}. 
In this work, we use LoRA for pretraining instead.
In particular, we take randomly initialized neural network weights $W_0\in\mathbb{R}^{a\times b}$ and represent their update during pretraining as $UV$, yielding $W_{\text{pretrained}}=W_{0}+\Delta W=W_{0}+UV$. 
We decrease the dimensionality further by applying subspace projection to the LoRA matrices, which we describe in detail in \Cref{sec:sublora}. 

\section{Methodology}
\label{sec:methodology}
In constructing non-vacuous generalization bounds for LLMs, we expand and improve upon existing techniques in three ways: (1) we construct a simple and effective nonlinear parameterization which is more effective and scalable than purely linear subspaces; (2) we construct new bounds that can handle the continuous and unbounded nature of the negative log-likelihood; (3) we make these bounds more practical to compute with LLMs by deriving a new bound which holds even when the empirical risk is evaluated only on a small subsample of the full training dataset.

\subsection{Finite Hypothesis Compression Based Generalization Bounds}
\label{sec:first-bounds}

Given a bounded risk $R(h,x) \in [a,a+\Delta]$ and a finite hypothesis space $h\in \mathcal{H}$ for which we have a prior $P(h)$, it is straightforward to derive a generalization bound relating the empirical risk \smash{$\hat{R}(h) = \frac{1}{m}\sum_{i=1}^m R(h,X_i)$} to the expected risk \smash{$R(h) = \mathbb{E}[\hat{R}(h)]$} so long as $\{X_i\}_{i=1}^m$ are sampled independently.
With probability at least $1-\delta$, as in  \citet{shalev2014understanding}, we can show that
\begin{align}\label{eq:bound}
    R(h) \le \hat{R}(h)+\Delta\sqrt{\frac{\log 1/P(h) + \log 1/\delta}{2m}}. 
\end{align}
We provide an elementary proof in \Cref{app:bound}.

If the prior likelihood $P(h)$ of the found model $h$ can be increased (either by choosing a better prior, or by finding more likely hypotheses), then the generalization bound improves. Following \citet{lotfi2022pac}, we adopt the powerful but general Solomonoff prior $P(h) \le 2^{-K(h|A)}$ \citep{solomonoff1964formal} where $K$ is the prefix Kolmogorov complexity of $h$, with the model architecture $A$ provided as input.
The Kolmogorov complexity of hypothesis $h$ is defined as the length of the shortest program that produces $h$ for a fixed programming language $P$ \citep{kolmogorov1963tables}.
While $K$ is not computable, it is possible to compute the upper bound 
\begin{align*}
    \log 1/P(h) \le K(h|A)\log 2 \le  C(h)\log2+2\log C(h),
\end{align*}
where $C(h)$ is the number of bits required to represent hypothesis h using some pre-specified coding.
Therefore, if we can find hypotheses $h$ that both have a low empirical risk \textit{and} a small compressed size, then we can construct strong generalization bounds even for large models that can represent many hypotheses. Indeed our bounds are often tighter for larger models.  We note that in contrast to standard PAC-Bayes bounds, these finite hypothesis bounds apply to deterministically trained models.

\subsection{Enabling the Independence Assumption for Generalization Bounds on Text Data}
Using \autoref{eq:bound} requires that $X_i$ in the sum  \smash{$\hat{R}(h) = \frac{1}{m}\sum_{i=1}^m R(h,X_i)$} are drawn independently. Thus, we must be careful in the construction and interpretation of our bounds so that this constraint is satisfied. Instead of considering bounds at the level of tokens, which are correlated, we instead define $X_i$ to be an entire document sampled from the data generating process from which the corpus was sampled. We define the risk on a given document as the negative log-likelihood of the entire document divided by its length, according to the autoregressive model.

It is also possible to choose $X_i$ to be a \emph{context chunk}, i.e., a sequence of length equal to the context length, as is commonly used in the training of models since a document may be larger than the maximum transformer context length.
In such cases, the sequences are no longer independent samples from the data generating process.
It is possible to construct valid bounds on these sequences which respect the independence assumption. However, in doing so we must shift the interpretation of the bounds from being over the randomness in sampling from the data generating process to the randomness in sampling sequences that can be constructed from a fixed and finite dataset formed by concatenating the documents together.

We explore these alternate sequence-level bounds in \autoref{app:chunk_bounds}. However, our focus in this paper is on document-level bounds.

\vspace*{3pt}
\begin{note}[Satisfying the I.I.D assumption] For document-level bounds, we sample the documents independently from the dataset. For sequence-level bounds, we divide the OpenWebText dataset into sequences of length equal to the context length and sample sequences independently so that a single sample from the dataset includes all of the tokens in the given sequence. Whether or not samples from a set are independent depends on how we sample from the set and not the contents of the set. 
We can draw independent samples from a set containing only similar elements, and the similarity of the elements does not affect the independence. Specifically, each element of our set is either an entire document or an entire sequence of tokens—not individual tokens—and we draw them independently.
\end{note}

\subsection{Accommodating the Unbounded NLL Objective Using Prediction Smoothing}

The primary metric for pretraining of large language models, as for other autoregressive models, is the negative log-likelihood (NLL), or bits per dimension (BPD), of the generative model.
The BPD loss is formally defined as the average over the negative log probabilities in logarithm base 2 as follows: $
\mathrm{BPD}(h, X) = - \frac{1}{k}\sum_i^k \log_2 p_h(x_i|x_{<i})$.
Unlike classification error which is a $\{0,1\}$ valued random variable, the log-likelihood is an unbounded quantity that does not have an obvious sub-Gaussian, or other, well-understood tail behavior.

To overcome this challenge, we construct generalization bounds for BPD not of the original model but instead on a smoothed version of it that limits the worst case behavior.
We define this smoothed model as a token-level mixture of the original LLM token predictions and a uniform distribution over the vocabulary of size $V$:\vspace*{-5pt}
\begin{align}\label{eq:mixture}
    p_h(x_i|x_{<i}) = (1-\alpha)p_{\theta}(x_i|x_{<i}) + \alpha/V ,
\end{align}\\[-15pt]
where $p_{\theta}(x_i|x_{<i})$ is the base model of token probabilities, $\alpha \in (0,1)$ is the mixing parameter, and $p_{h}(x_i|x_{<i})$ is the smoothed predictor.

The model on an entire document $X$ composed of $L$ tokens is defined autoregressively in terms of this mixture model $p_h(X) := \Pi_i^L p_h(x_i|x_{<i})$, and we find this to be a more effective way of constructing the bounds than constructing the mixture at the document level.
In analogy to label smoothing where the labels of the training objective are mixed with the uniform distribution, we term this operation as prediction smoothing.

As we show in \Cref{app:bounded}, the NLL of the prediction smoothed model on a document $\mathrm{BPD}(h, X):= -\log_2 p_h(X)/L$ can be bounded as follows:\vspace*{-5pt}
\begin{align*}
    \log_2 (V/\alpha) - \Delta & \le \mathrm{BPD}(h,X) \le \log_2 (V/\alpha),
\end{align*}\\[-15pt]
for $\Delta= \log_2 \big(1+(1-\alpha)V/\alpha\big)$.
With prediction smoothing, the risk $R(h,X)=\mathrm{BPD}(h, X)$ on a given document is bounded in an interval of size $\Delta$, and therefore we can use \Cref{eq:bound} to generate bounds for negative log-likelihood of this model. 
We refer to $\Delta$ as the worst-case interval size.

We explore the trade-off over different values of $\alpha$ in \Cref{fig2:effect-alpha} (right).
As $\alpha$ gets larger, the interval size $\Delta$ representing the worst-case behavior goes down, whereas the empirical risk goes up, leading to a sweet spot in the middle.
By defining the hypothesis $h=(\theta, d, r, \alpha)$ to include the model parameters, LoRA space hyperparameters $d,r$, and the mixture weight $\alpha$, we can view $\alpha$ as merely one additional model parameter accounted in $\log 1/P(h)$.
By doing so, we are free to optimize over $\alpha$ in the computation of the bound, and we can do so without retraining the model.

\begin{figure}[t!]
\centering
\includegraphics[width=1.0\linewidth]{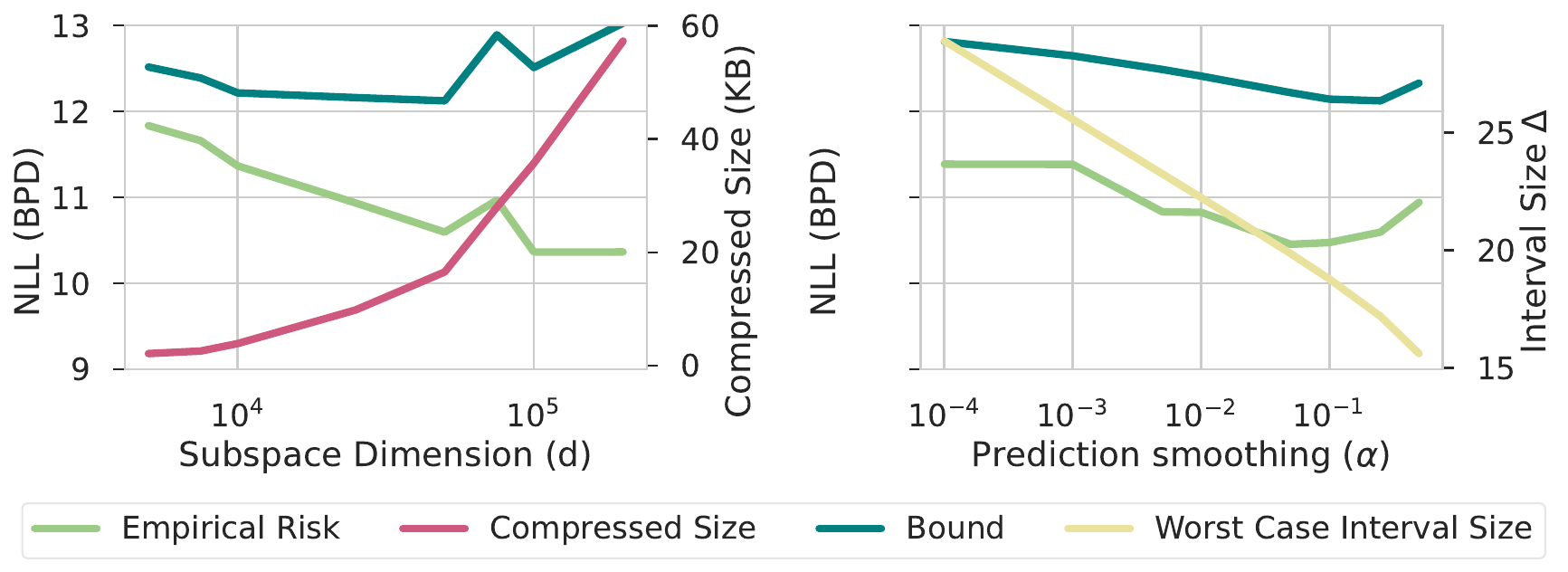}
    \caption{\textbf{Varying Parameters of the Compression Bounds.}
    \textbf{(Left):} A plot of the generalization bound as a function of the projection dimension $d$ with LoRA. The subspace dimension gives us a way to explicitly trade off the degree of compression with the empirical risk, and we optimize $d$ to produce the best bounds.
    \textbf{(Right):} A plot of the worst case range of BPD values $\Delta$, empirical risk, and the resulting generalization bounds as a function of the prediction smoothing parameter $\alpha$. For each model, a different alpha can be chosen after the models have already been trained.} 
    \label{fig2:effect-alpha}
    \vspace*{-8pt}
\end{figure}

\setlength{\tabcolsep}{9.0pt}
\begin{table*}[t!]
\centering
\caption{\textbf{Non-vacuous generalization bounds for GPT-2 compressed models.} Our best document-level generalization bounds achieved for the GPT-2 architecture for BPD and Top-k token prediction error, all of which are non-vacuous.}
\vspace*{5pt}
\begin{tabular}{l|c|c|c|c|c}
\toprule
Metric & SubLoRA & LoRA Only & Subspace Only & Original Model & Random Guess \\ 
\midrule
Top-1 Error Bound (\%) & $\mathbf{96.41}$ & 100 & 96.52 & 100 & 99.99  \\
Top-10 Error Bound (\%) & $\mathbf{77.90}$ & 84.37   & 79.36 & 100 & 99.98 \\
Top-100 Error Bound (\%) & $\mathbf{58.34}$ & 67.26   & 75.95 & 100 & 99.80  \\
Bits per Dimension Bound & $\mathbf{12.12}$ & 13.09   & 14.59 & 70.76 & 15.62  \\
\bottomrule
\end{tabular}
\label{tab:error-comparison-scratch}
\end{table*}

\subsection{Using Subsampling in Bound Computation}
\label{sec:final-bounds}

The empirical risk requires evaluating the model on the full training dataset of $m$ data points: \smash{$\hat{R}(h) = \frac{1}{m}\sum_{i=1}\hat{R}_i(h)$}.
As large language models are typically trained for only $1$ epoch or less, doing so is prohibitively expensive. Instead, we propose to modify our generalization bounds to account for evaluating only a subsample of size $n \ll m$ of the training dataset when computing the empirical risk.

Denoting \smash{$\hat{\hat{R}}(h) = \sum_{i=1}^n \hat{R}_{\sigma(i)}(h)$} where $\sigma(i)$ is a random sample (with replacement) from $1, \dots,  m$. In \Cref{app:subsampling} we derive a new bound both over the randomness in $\sigma(i)$ and the randomness in $X$ which holds with probability $\ge 1-\delta$:
\begin{align}\label{eq:subset_bound}
    R(h) \le \hat{\hat{R}}(h)+\Delta\sqrt{\frac{\log \tfrac{1}{P(h)} + \log \tfrac{1}{s\delta}}{2m}} + \Delta \sqrt{\frac{\log \tfrac{1}{(1-s)\delta}}{2n}} ,
\end{align}\\[-10pt]
where $s=n/(n+m)$.  Using this subsampling bound, we can accelerate bound computation.
For dataset sizes in the $10$'s of millions, we can get away with evaluating only $10,000$ data points after the model has been trained, with a negligible penalty in the bounds.
In fact, we need not even train on the entirety of the training data in order to produce valid bounds as long we sample uniformly. 

\section{SubLoRA: A Simple and Efficient Nonlinear Parameterization of the Hypothesis Space}
\label{sec:sublora}

To find compressible solutions $h$ that simultaneously are expressive enough to achieve low training error, we search over a carefully designed manifold of possible parameters that live within the parameter space.

In contrast to \citet{lotfi2022pac}, we consider a nonlinear parameterization of the model weights $\theta = f(\theta_0, w)$ given by the composition of LoRA \citep{hu2021lora} (a nonlinear parameterization) and the subspace compression matrices. 
Given a vector of model parameters $\theta$, we break down its constituent components into the different weight matrices $W_i$ and associated biases $b_i$: $\mathrm{unflatten}(\theta) = \{(W_i, b_i)\}_{i\in I}$.
We define a nonlinear parameterization of the hypothesis space as\vspace*{-8pt}
\begin{align}
\label{eq:non-linear}
    \theta = \theta_0+\mathrm{LoRA}(Pw),
\end{align}\\[-12pt]
where $\mathrm{LoRA}$ is defined by the implementation of the low-rank products for the weight matrices, leaving the biases unchanged.
As $Pw$ and $\theta$ are the flattened parameter vectors, $\mathrm{LoRA}(\cdot)$ is defined as the operation that unflattens the vector, applies the low-rank product, and then flattens the result. 
Here, $\theta_0$ is merely a random initialization of the model parameters, and $P \in \mathbb{R}^{D\times d}$ is a Kronecker product projector \smash{$P=Q_1 \otimes Q_2$} for $Q_1,Q_2$ constructed by orthogonalizing Gaussian random matrices by QR factorization: $P_1,P_2 \sim \mathcal{N}(0,1/\sqrt{D})^{\sqrt{D}\times \sqrt{d}}$ with $Q_1 R_1 = P_1$ and similarly for $Q_2$.
We apply LoRA only over the self-attention layer and the last linear layer weight matrices, meaning that other model parameters do not differ from their initialized values.
In order to compress the model, we need only to represent the vector $w$ since $\theta_0$ and $P$ are chosen ahead of time and specified in the architecture via random initialization.

\begin{note}[Selecting the LoRA layers]
While LoRA was developed for finetuning LLMs, we find that even when pretraining using LoRA, we can achieve non-trivial performance. 
Our initial exploration of LoRA for pretraining involved applying LoRA not only to attention layers but to all other linear layers as well. 
We found that for pretraining, it is more efficient to use LoRA for both the attention layers and the last linear layer, while including other layers provides insignificant returns.
\end{note}

In \Cref{fig:fig1} (left), we show the Pareto frontier of empirical risk and the complexity penalty in the relevant generalization bound with LoRA, subspace training, and SubLoRA.
Rather than being competing methods for compression, LoRA and subspace training are complementary and exploit different structures in the parameter space to provide a family of models in the original hypothesis space that are both expressive and compressible. 
SubLoRA achieves a strict improvement over LoRA and subspace training, often being the deciding factor whether the bounds are vacuous or non-vacuous.
In \Cref{fig2:effect-alpha} (left), we explore how the compressed size of the model and the empirical risk vary as a function of the subspace dimension $d$.

\setlength{\tabcolsep}{9.0pt}
\begin{table*}[h!]
\centering
\caption{\textbf{Validation performance vs. bounds.} Validation performance of the models achieving the best bits per dimension bound for each setting. For models trained using different compression techniques, SubLoRA performs best in terms of validation loss. Although the original model achieves the best validation BPD and NLL, it leads to vacuous bounds given its large number of parameters.}
\vspace*{5pt}
\begin{tabular}{l|c|c|c|c|c}
\toprule
Metric & SubLoRA & LoRA Only & Subspace Only & Original Model & Random Guess \\ 
\midrule
Bits per Dimension Bound & $12.12$ & 13.09   & 14.59 & 70.76 & 15.62  \\
Validation BPD &   10.53  &    11.21      &  13.85    &   4.35 & --  \\
Validation NLL &  7.29     &  7.77   &  9.60   &  3.01   &-- \\
\bottomrule
\end{tabular}
\label{apptab:val-perf}
\end{table*}

\section{Non-Vacuous Generalization Bounds for Large Language Models}

We outline the pretraining and bound computation pipeline and then present our empirical results.

\subsection{End-to-end Pipeline} 
Assembling the components described in \Cref{sec:methodology}, we train variants of a GPT-style architecture through the nonlinear compressed parameterization in \Cref{eq:non-linear}.
We use several values for the subspace dimension $d$ and two values for the rank of the LoRA matrices $r$.
Nearing the end of training, we train for additional steps using quantization-aware training with a small number of quantization levels (with additional details listed in \Cref{app:exp-details}). 
We express $w$ in this quantization and encode it using arithmetic coding to determine the compressed size of the model. 
Added to the size of the model are the bits needed to encode the choice of $d,r$, $\alpha$, the learning rate, and the quantization levels. 

We evaluate the empirical log probabilities and token predictions for each token in the document on a small subset of the training data $n=10,000$ documents. 
We use this sub-sampling size to evaluate all the bounds reported in the paper.
With these predictions, we can compute the generalization bound in \Cref{eq:subset_bound} as a function of $\alpha$, and we optimize over this parameter for each model.
Finally, we can tune the extent of compression through the different choices of $d$ and choose the subspace dimension that produces the best bound. 
We provide a pseudo-code and additional details about the bound computation in \Cref{app-sec:bound-computation}. 

\subsection{Non-Vacuous Bounds for GPT-2 Small}
\vspace*{-4pt}

We consider the GPT-2 small architecture with $124$M parameters and compute our next token prediction document-level bounds by pretraining these models on the OpenWebText dataset using SubLoRA.
We report the results in \Cref{tab:error-comparison-scratch}. 
We consider the token level error averaged over a document as the empirical risk.
For instance, the Top-1 Error Bound refers to the upper bound on the expected Top-1 error per token averaged over the document $R(h,X_k) = \frac{1}{L} \sum_{i=1}^L \mathbf{1}{[\operatorname{argmax_{x_i}} p(x_{i}|x_{<i}=x_{<i}^k) =x_{i}^k]}$, where the upper index $k$ denotes the document index and the lower index $i$ denotes the position within the document. 
The $\operatorname{argmax}$ operator operates over tokens $x_i$ across the vocabulary. 
Thus, $R(h,X_k)$ represents the proportion of tokens accurately predicted within a given document.
Random guess performance is $\log_2 V$ for BPD and $1-k/V$ for Top-k Error.

The best bounds are indeed obtained using our simple compression technique, which combines the strengths of both low-rank adaptation and subspace training. 
When we solely apply quantization and arithmetic coding without implementing LoRA or linear subspace compression during the training phase, we obtain vacuous bounds.

\subsection{Validation Performance}

In \cref{apptab:val-perf}, we report the validation performance of the models that achieve the best bounds in \cref{tab:error-comparison-scratch}. In particular, we report the bits per dimension (BPD) and negative log-likelihood (NLL) losses alongside the generalization bounds. For compressed models, we see that the bounds have the same trend as the validation performance, as all these models contain a smaller number of parameters compared to the initial parameter space. Although the original model with the full set of parameters achieves the best validation performance, it leads to vacuous bounds given its large number of parameters.

It is important to note that we report the validation BPD loss of the model achieving the best bounds, and not the best validation BPD loss achieved by the different compression schemes, which can be lower depending on the subspace parameters. \cref{fig2:effect-alpha} (left) reflects this trade-off between the compressed size of the model and the empirical risk.

\begin{figure*}[t!]
\centering
    \includegraphics[width=0.87\linewidth]{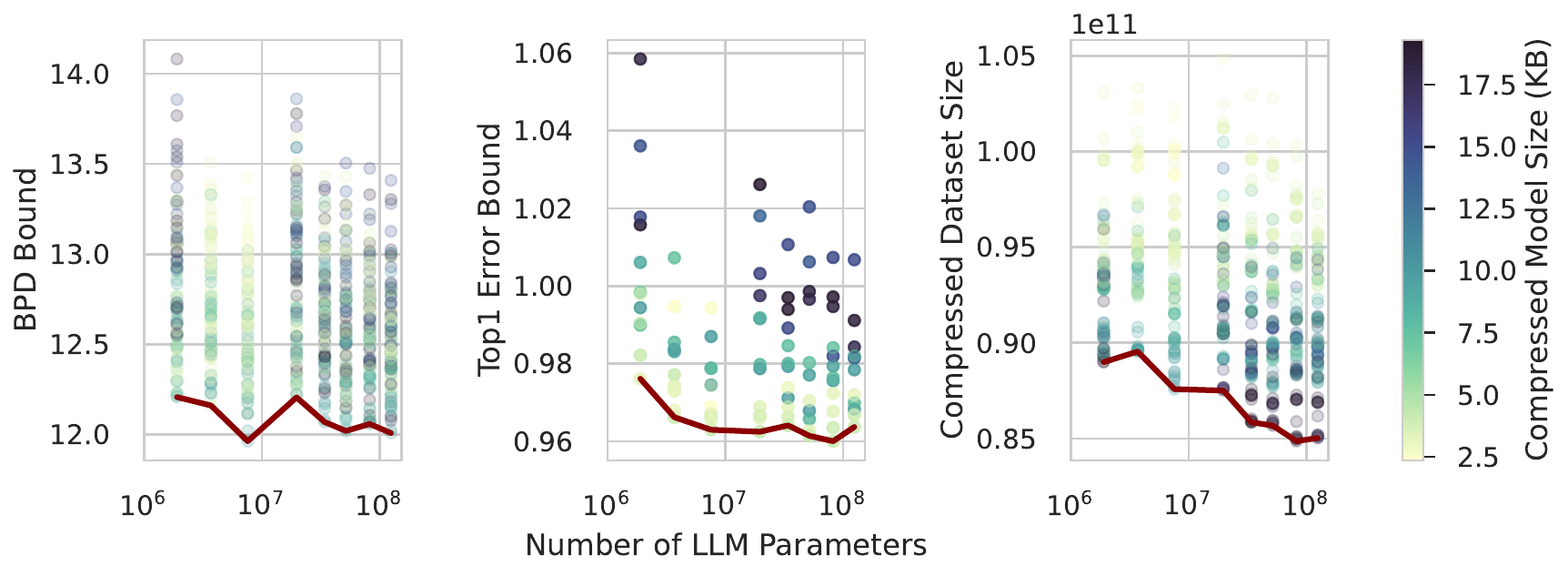}
    \vspace*{-7pt}
    \caption{\textbf{Larger models achieve stronger generalization bounds.} As we scale up the size of the model via the model parameters (holding the training set fixed), we find that our generalization bounds get \emph{better} rather than worse. Dots show models trained with differing degrees of compression, indicated by their color. On the right we show the number of bits required to express the training dataset using the model and including the model weights in the compression. Classification error bounds consistently favor smaller models, while data compression favors much larger models, and BPD bounds are in between.}
    \label{fig:llms-scale-vs-compression}
\end{figure*}

\vspace*{-5pt}
\subsection{Extending Our Bounds to Larger Models}
\vspace*{-4pt}

We use SubLoRA to obtain generalization bounds for much larger variants of GPT-2 of sizes $354$M (GPT-2 medium), $458$M, $773$M (GPT-2 large), and $849$M parameters.  
\Cref{tab:larger-model-bounds} shows that our simple compression approach yields non-vacuous bounds for models with nearly a billion parameters. 
Moreover, we see that the smallest model, where we previously performed experiments and tuned our hyperparameters, actually achieves the worst bound on bits per dimension as we scale the models up.
In conclusion, our approach extends naturally to much larger language models and proves that it is possible to achieve tighter bounds as we increase the size of the model.

\begin{table}[h!]
\vspace*{-8pt}
\centering
\caption{\textbf{Non-vacuous generalization bounds for models with up to 849M parameters.} Non-vacuous bounds achieved for GPT-2 architectures with different sizes, ranging from $124$ to $849$ million parameters. We report below the bounds on the bits-per-dimension (BPD), Top-1 Error, and Top-100 Error.
All of the BPD bounds are non-vacuous and tighter than the GPT-2 small bounds.}
\vspace*{5pt}
\resizebox{\columnwidth}{!}{
\begin{tabular}{l|c|c|c}
\toprule
Model Size & BPD & Top-1 Error & Top-100 Error\\ 
\midrule
$124$M & $12.12$ & $96.41$ & $58.34$  \\
$354$M & $11.96$ & $95.99$   & $58.4$ \\
$458$M & $11.95$ & $96.69$  & $58.49$  \\
$773$M & $12.10$ & $96.17$  & $59.25$  \\
$849$M & $12.01$ & $96.51$  & $58.89$  \\
\bottomrule
\end{tabular}
}
\label{tab:larger-model-bounds}
\end{table}

\vspace*{3pt}
\begin{note}[Limitations]
Note that due to computational constraints, we pre-train the larger GPT-2 variants with SubLoRA only for a limited number of hyperparameter settings in contrast to the $124$M model for which we did a thorough hyperparameter sweep. It is likely that the tightest empirically achievable bounds are much stronger for the new large models than what we report in \Cref{tab:larger-model-bounds}.
\end{note}

\section{Understanding the Generalization of LLMs}
\label{sec:understanding-llms}

As language models grow in size, it is clear that they gain an increasing capacity to fit their training data.
On the one hand, this increasing capacity might mean that, as LLMs become capable of learning increasingly complex functions, they become increasingly likely to merely memorize their training samples and not perform any meaningful generalization beyond their training corpora.
On the other hand, large language models have proven to be surprisingly capable of generalizing, often extending to tasks that seem quite different from the training objective.

We investigate the tension between these two narratives along several fronts: We assess how generalization bounds change with the size of the model, whether language models can form a compression of the training data even when accounting for their large size, and how structure in the training data affects the generalization of the learned model.
In \Cref{app:llms-finetuning}, we use our bounds to quantify of the benefits of pre-training in LLMs.

\subsection{Larger Models Are More Compressible and Generalize Better}
Empirically, it has been found that LLMs generalize better as the number of parameters is increased, with a fixed size of dataset \citep{kaplan2020scaling,brown2020language}, and this fact is of great importance leading to the creation of ever larger and more powerful models.
From a generalization theory perspective, this trend is counterintuitive because of the growing hypothesis class, and a naive analysis would suggest that larger models should generalize worse. To date, we are not aware of any convincing demonstration that generalization bounds improve with more parameters on models of practical sizes.

We evaluate our bounds on a collection of LLMs with different numbers of parameters, choosing the appropriate scaling for the width, depth, number of attention heads, etc.
Surprisingly, we find that our generalization bounds in fact \emph{improve} with model size, even as the training dataset is held fixed.
With our SubLoRA compression, larger models find even simpler representations of the data given a fixed training set.
These results are shown in \Cref{fig:llms-scale-vs-compression}.
While some explanations for why larger models should generalize better have been put forward in the literature \citep{nakkiran2021deep,gunasekar2017implicit}, the mechanism by which larger models become more compressible is not clear, and we believe this result is noteworthy and requires further investigation.

In addition to constructing generalization bounds, we can use our compressed models to form a compression of the training dataset itself. In \autoref{fig:llms-scale-vs-compression}, we count the number of bits needed to encode the model $C(h)$ and the number of bits to encode the data using the model $C(\{X\}_{i=1}^m|h)$, which is the negative log-likelihood of the entire dataset according to the model. 
Adding these two up, we have a compression of the training dataset using the model, and one which is closely related to our generalization bounds.

\subsection{How Does Generalization of Large Language Models Depend on Structure in Text?}
Neural networks that fit a training dataset of random noise will not be able to generalize, and the ability of overparametrized networks to fit noise implies that uniform convergence is impossible across the general hypothesis class \citep{nagarajan2019uniform}. 
This is a clear demonstration that the structure of the dataset influences generalization. However, the impact of more subtle structures is less understood theoretically. Here, we use our bounds to investigate how the temporal order structure relates to generalization. 

We train models that explicitly break the temporal structure of the text data by applying random permutations to each sequence during training. Consequently, the model can only make use of the input information as if it were a bag of words. We find that this broken order structure indeed leads to less favorable generalization bounds. 
\autoref{fig:permute} shows the best error bounds when the original and perturbed data are used to train the model and evaluate the bounds for the bits per dimension, top-1 error, and top-100 error losses.
While the top-1 error bound becomes vacuous as we break the text structure, the top-100 error and bits per dimensions bounds remain non-vacuous.
This might be due to the fact that as we perturb the sequence, predicting the next token accurately becomes an extremely difficult task for LLMs, while predicting a token that fits generally into the context, without necessarily being the correct token, is an easier task.

\begin{figure}[h!]
\vspace*{-5pt}
\centering
\includegraphics[width=.4\textwidth]{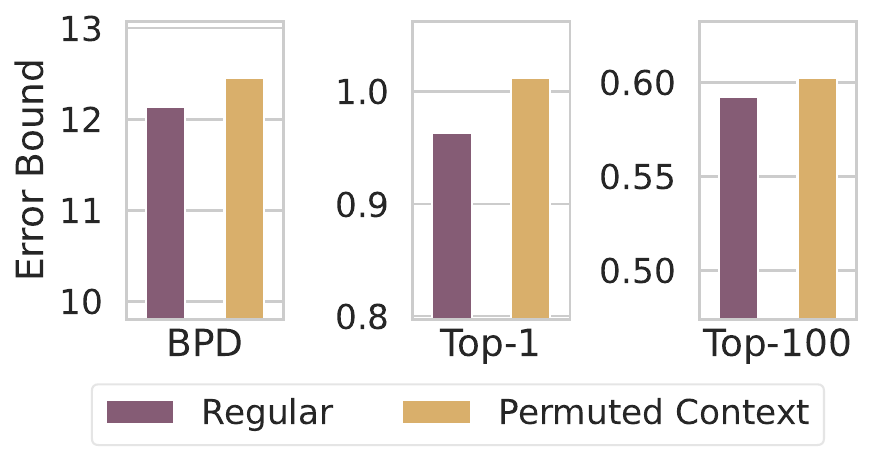}
\vspace*{-8pt}
\caption{\textbf{Breaking text structure with permutations.} We compute bounds for LLMs that were trained with the order of the tokens shuffled within each sequence.}
\label{fig:permute}
\vspace*{-10pt}
\end{figure}

\section{Discussion}

In this work, we demonstrated that large language models can themselves form highly compressed representations of distributions over text.
Using these highly compressed LLMs, we compute the first non-vacuous generalization bounds for LLM pretraining.
Our findings suggest that the development of tighter compression bounds presents a fruitful avenue for understanding how and why language models generalize.

We discuss below the limitations of our work and their implications for future work.

\textbf{Non I.I.D. token level bounds.}
In our work, we split up the training data into I.I.D. documents that form the basis of our bounds.
However, the loss for each of these documents also decomposes as a sum over non I.I.D. tokens, and it is likely that this additional structure could also be exploited in the bound construction to significantly increase the effective number of training samples.

\textbf{Efficient bound computation on pretrained models.}
Our procedure for computing generalization bounds requires training LLMs from scratch through our SubLoRA parametrization. It may be possible to devise a fast method of computing bounds on a model that has already been trained, but still constraining its generalization error.

\textbf{Nonlinear parameterizations.} Unlike previous state-of-the-art bounds from \citet{lotfi2022pac}, we employ a nonlinear parameterization via LoRA, significantly improving the bounds.  This observation opens up an avenue for rich nonlinear parameterizations that simultaneously reduce the number of parameters while also including diverse functions which are likely to fit the training data. 

\textbf{Bounds for models that generate high quality text.}
In \Cref{tab:gpt2text} and \Cref{tab:gpt2subloratext}, we show samples of generated text using both a GPT-2 style model pretrained in the standard fashion and a GPT-2 style model pretrained using SubLoRA.
While the vanilla GPT-2 style model produces reasonable sentences, the SubLoRA pretrained model often outputs ungrammatical text. 

\textbf{Alternative approaches to learning with LLMs.} Modern language models make possible new inference techniques such as in-context learning and prompt-tuning. These modes are already seeing widespread deployment and warrant analogous theories of generalization.

\textbf{Generalization beyond the training distribution.}
Recent work showed that language models prefer low-complexity numerical sequences on which they were not trained, even at random initialization \citep{goldblum2023no}, and generalization theory may be useful for explaining why LLMs can generalize far outside of their training distribution, and even outside of the text modality, for example to tabular data \citep{hegselmann2023tabllm} or images \citep{deletang2023language}.

We hope that future work will address these limitations and that our approach can help lay the groundwork for computing LLM generalization bounds for even larger models.

\section*{Acknowledgements}
\label{sec: acknowledge}

We acknowledge anonymous reviewers for helpful feedback. This work is supported by NSF CAREER IIS-2145492,
NSF CDS\&E-MSS 2134216, NSF HDR-2118310, BigHat Biosciences, Capital One, and an Amazon Research Award.

\section*{Impact Statement}
\label{sec:impact}

The goal of this work is to advance our understanding of large language models and improve their trustworthiness.
However, we note that there are many facets of large language models---such as biases---that may not be captured by generalization bounds.

\bibliography{references}

\begin{thebibliography}{43}
\providecommand{\natexlab}[1]{#1}
\providecommand{\url}[1]{\texttt{#1}}
\expandafter\ifx\csname urlstyle\endcsname\relax
  \providecommand{\doi}[1]{doi: #1}\else
  \providecommand{\doi}{doi: \begingroup \urlstyle{rm}\Url}\fi

\bibitem[Aghajanyan et~al.(2020)Aghajanyan, Zettlemoyer, and
  Gupta]{aghajanyan2020intrinsic}
Aghajanyan, A., Zettlemoyer, L., and Gupta, S.
\newblock Intrinsic dimensionality explains the effectiveness of language model
  fine-tuning.
\newblock \emph{arXiv preprint arXiv:2012.13255}, 2020.

\bibitem[Alquier et~al.(2016)Alquier, Ridgway, and
  Chopin]{alquier2016properties}
Alquier, P., Ridgway, J., and Chopin, N.
\newblock On the properties of variational approximations of gibbs posteriors.
\newblock \emph{The Journal of Machine Learning Research}, 17\penalty0
  (1):\penalty0 8374--8414, 2016.

\bibitem[Bartlett \& Mendelson(2002)Bartlett and
  Mendelson]{bartlett2002rademacher}
Bartlett, P.~L. and Mendelson, S.
\newblock Rademacher and gaussian complexities: Risk bounds and structural
  results.
\newblock \emph{Journal of Machine Learning Research}, 3\penalty0
  (Nov):\penalty0 463--482, 2002.

\bibitem[Bharadwaj \& Hasegawa-Johnson(2014)Bharadwaj and
  Hasegawa-Johnson]{bharadwaj2014pac}
Bharadwaj, S. and Hasegawa-Johnson, M.
\newblock A {PAC}-{B}ayesian approach to minimum perplexity language modeling.
\newblock In \emph{Proceedings of {COLING} 2014, the 25th International
  Conference on Computational Linguistics: Technical Papers}, pp.\  130--140,
  Dublin, Ireland, 2014.

\bibitem[Brown et~al.(2020)Brown, Mann, Ryder, Subbiah, Kaplan, Dhariwal,
  Neelakantan, Shyam, Sastry, Askell, et~al.]{brown2020language}
Brown, T., Mann, B., Ryder, N., Subbiah, M., Kaplan, J.~D., Dhariwal, P.,
  Neelakantan, A., Shyam, P., Sastry, G., Askell, A., et~al.
\newblock Language models are few-shot learners.
\newblock \emph{Advances in neural information processing systems},
  33:\penalty0 1877--1901, 2020.

\bibitem[Carlini et~al.(2020)Carlini, Tramèr, Wallace, Jagielski,
  Herbert-Voss, Lee, Roberts, Brown, Song, Erlingsson, Oprea, and
  Raffel]{carlini2020extracting}
Carlini, N., Tramèr, F., Wallace, E., Jagielski, M., Herbert-Voss, A., Lee,
  K., Roberts, A., Brown, T.~B., Song, D., Erlingsson, U., Oprea, A., and
  Raffel, C.
\newblock Extracting training data from large language models.
\newblock \emph{arXiv preprint arXiv:2012.07805}, 2020.

\bibitem[Carlini et~al.(2023)Carlini, Ippolito, Jagielski, Lee, Tramer, and
  Zhang]{carlini2023quantifyingiclr}
Carlini, N., Ippolito, D., Jagielski, M., Lee, K., Tramer, F., and Zhang, C.
\newblock Quantifying memorization across neural language models.
\newblock \emph{Proceedings of the 37th International Conference on Learning
  Representations (ICLR 2023)}, 2023.

\bibitem[Catoni(2007)]{catoni2007pac}
Catoni, O.
\newblock Pac-bayesian supervised classification: the thermodynamics of
  statistical learning.
\newblock \emph{arXiv preprint arXiv:0712.0248}, 2007.

\bibitem[Chowdhery et~al.(2022)Chowdhery, Narang, Devlin, Bosma, Mishra,
  Roberts, Barham, Chung, Sutton, Gehrmann, Schuh, Shi, Tsvyashchenko, Maynez,
  Rao, Barnes, Tay, Shazeer, Prabhakaran, Reif, Du, Hutchinson, Pope, Bradbury,
  Austin, Isard, Gur-Ari, Yin, Duke, Levskaya, Ghemawat, Dev, Michalewski,
  Garcia, Misra, Robinson, Fedus, Zhou, Ippolito, Luan, Lim, Zoph, Spiridonov,
  Sepassi, Dohan, Agrawal, Omernick, Dai, Pillai, Pellat, Lewkowycz, Moreira,
  Child, Polozov, Lee, Zhou, Wang, Saeta, Diaz, Firat, Catasta, Wei,
  Meier-Hellstern, Eck, Dean, Petrov, and Fiedel]{palm}
Chowdhery, A., Narang, S., Devlin, J., Bosma, M., Mishra, G., Roberts, A.,
  Barham, P., Chung, H.~W., Sutton, C., Gehrmann, S., Schuh, P., Shi, K.,
  Tsvyashchenko, S., Maynez, J., Rao, A., Barnes, P., Tay, Y., Shazeer, N.,
  Prabhakaran, V., Reif, E., Du, N., Hutchinson, B., Pope, R., Bradbury, J.,
  Austin, J., Isard, M., Gur-Ari, G., Yin, P., Duke, T., Levskaya, A.,
  Ghemawat, S., Dev, S., Michalewski, H., Garcia, X., Misra, V., Robinson, K.,
  Fedus, L., Zhou, D., Ippolito, D., Luan, D., Lim, H., Zoph, B., Spiridonov,
  A., Sepassi, R., Dohan, D., Agrawal, S., Omernick, M., Dai, A.~M., Pillai,
  T.~S., Pellat, M., Lewkowycz, A., Moreira, E., Child, R., Polozov, O., Lee,
  K., Zhou, Z., Wang, X., Saeta, B., Diaz, M., Firat, O., Catasta, M., Wei, J.,
  Meier-Hellstern, K., Eck, D., Dean, J., Petrov, S., and Fiedel, N.
\newblock Palm: Scaling language modeling with pathways, 2022.

\bibitem[Del{\'e}tang et~al.(2023)Del{\'e}tang, Ruoss, Duquenne, Catt,
  Genewein, Mattern, Grau-Moya, Wenliang, Aitchison, Orseau,
  et~al.]{deletang2023language}
Del{\'e}tang, G., Ruoss, A., Duquenne, P.-A., Catt, E., Genewein, T., Mattern,
  C., Grau-Moya, J., Wenliang, L.~K., Aitchison, M., Orseau, L., et~al.
\newblock Language modeling is compression.
\newblock \emph{arXiv preprint arXiv:2309.10668}, 2023.

\bibitem[Dettmers et~al.(2023{\natexlab{a}})Dettmers, Shmitchell, Roberts, Lee,
  Brown, Song, and Raffel]{dettmers2023qlora}
Dettmers, T., Shmitchell, S., Roberts, A., Lee, K., Brown, T.~B., Song, D., and
  Raffel, C.
\newblock Qlora: Efficient finetuning of quantized llms.
\newblock \emph{arXiv preprint arXiv:2305.14314}, 2023{\natexlab{a}}.

\bibitem[Dettmers et~al.(2023{\natexlab{b}})Dettmers, Shmitchell, Roberts, Lee,
  Brown, Song, and Raffel]{dettmers2023spqr}
Dettmers, T., Shmitchell, S., Roberts, A., Lee, K., Brown, T.~B., Song, D., and
  Raffel, C.
\newblock Spqr: A sparse-quantized representation for near-lossless llm weight
  compression.
\newblock \emph{arXiv preprint arXiv:2308.07234}, 2023{\natexlab{b}}.

\bibitem[Dziugaite \& Roy(2017)Dziugaite and Roy]{dziugaite2017computing}
Dziugaite, G.~K. and Roy, D.~M.
\newblock Computing nonvacuous generalization bounds for deep (stochastic)
  neural networks with many more parameters than training data.
\newblock \emph{arXiv preprint arXiv:1703.11008}, 2017.

\bibitem[Frantal et~al.(2022)Frantal, Gruslys, and Kiela]{frantal2022gptq}
Frantal, Z., Gruslys, A., and Kiela, D.
\newblock Gptq: Accurate post-training quantization for generative pre-trained
  transformers.
\newblock \emph{arXiv preprint arXiv:2210.17323}, 2022.

\bibitem[Germain et~al.(2016)Germain, Bach, Lacoste, and
  Lacoste-Julien]{germain2016pac}
Germain, P., Bach, F., Lacoste, A., and Lacoste-Julien, S.
\newblock Pac-bayesian theory meets bayesian inference.
\newblock \emph{Advances in Neural Information Processing Systems}, 29, 2016.

\bibitem[Goldblum et~al.(2023)Goldblum, Finzi, Rowan, and
  Wilson]{goldblum2023no}
Goldblum, M., Finzi, M., Rowan, K., and Wilson, A.~G.
\newblock The no free lunch theorem, kolmogorov complexity, and the role of
  inductive biases in machine learning.
\newblock \emph{arXiv preprint arXiv:2304.05366}, 2023.

\bibitem[Gunasekar et~al.(2017)Gunasekar, Woodworth, Bhojanapalli, Neyshabur,
  and Srebro]{gunasekar2017implicit}
Gunasekar, S., Woodworth, B.~E., Bhojanapalli, S., Neyshabur, B., and Srebro,
  N.
\newblock Implicit regularization in matrix factorization.
\newblock \emph{Advances in neural information processing systems}, 30, 2017.

\bibitem[Haddouche et~al.(2021)Haddouche, Guedj, Rivasplata, and
  Shawe-Taylor]{haddouche2021pac}
Haddouche, M., Guedj, B., Rivasplata, O., and Shawe-Taylor, J.
\newblock Pac-bayes unleashed: Generalisation bounds with unbounded losses.
\newblock \emph{Entropy}, 23\penalty0 (10):\penalty0 1330, 2021.

\bibitem[Hegselmann et~al.(2023)Hegselmann, Buendia, Lang, Agrawal, Jiang, and
  Sontag]{hegselmann2023tabllm}
Hegselmann, S., Buendia, A., Lang, H., Agrawal, M., Jiang, X., and Sontag, D.
\newblock Tabllm: Few-shot classification of tabular data with large language
  models.
\newblock In \emph{International Conference on Artificial Intelligence and
  Statistics}, pp.\  5549--5581. PMLR, 2023.

\bibitem[Hoeffding(1994)]{hoeffding1994probability}
Hoeffding, W.
\newblock Probability inequalities for sums of bounded random variables.
\newblock \emph{The collected works of Wassily Hoeffding}, pp.\  409--426,
  1994.

\bibitem[Holland(2019)]{holland2019pac}
Holland, M.
\newblock Pac-bayes under potentially heavy tails.
\newblock \emph{Advances in Neural Information Processing Systems}, 32, 2019.

\bibitem[Hu et~al.(2021)Hu, Shen, Wallis, Allen-Zhu, Li, Wang, Wang, and
  Chen]{hu2021lora}
Hu, E.~J., Shen, Y., Wallis, P., Allen-Zhu, Z., Li, Y., Wang, S., Wang, L., and
  Chen, W.
\newblock Lora: Low-rank adaptation of large language models.
\newblock \emph{arXiv preprint arXiv:2106.09685}, 2021.

\bibitem[Kaplan et~al.(2020)Kaplan, McCandlish, Henighan, Brown, Chess, Child,
  Gray, Radford, Wu, and Amodei]{kaplan2020scaling}
Kaplan, J., McCandlish, S., Henighan, T., Brown, T.~B., Chess, B., Child, R.,
  Gray, S., Radford, A., Wu, J., and Amodei, D.
\newblock Scaling laws for neural language models.
\newblock \emph{arXiv preprint arXiv:2001.08361}, 2020.

\bibitem[Kim et~al.(2023)Kim, Lee, Kim, Park, Yoo, Kwon, and
  Lee]{kim2023memory}
Kim, J., Lee, J.~H., Kim, S., Park, J., Yoo, K.~M., Kwon, S.~J., and Lee, D.
\newblock Memory-efficient fine-tuning of compressed large language models via
  sub-4-bit integer quantization.
\newblock \emph{arXiv preprint arXiv:2305.14152}, 2023.

\bibitem[Kolmogorov(1963)]{kolmogorov1963tables}
Kolmogorov, A.~N.
\newblock On tables of random numbers.
\newblock \emph{Sankhy{\=a}: The Indian Journal of Statistics, Series A}, pp.\
  369--376, 1963.

\bibitem[Kuzborskij \& Szepesv{\'a}ri(2019)Kuzborskij and
  Szepesv{\'a}ri]{kuzborskij2019efron}
Kuzborskij, I. and Szepesv{\'a}ri, C.
\newblock Efron-stein pac-bayesian inequalities.
\newblock \emph{arXiv preprint arXiv:1909.01931}, 2019.

\bibitem[Langdon(1984)]{langdon1984introduction}
Langdon, G.~G.
\newblock An introduction to arithmetic coding.
\newblock \emph{IBM Journal of Research and Development}, 28\penalty0
  (2):\penalty0 135--149, 1984.

\bibitem[Li et~al.(2018)Li, Farkhoor, Liu, and Yosinski]{li2018measuring}
Li, C., Farkhoor, H., Liu, R., and Yosinski, J.
\newblock Measuring the intrinsic dimension of objective landscapes.
\newblock \emph{arXiv preprint arXiv:1804.08838}, 2018.

\bibitem[Liu et~al.(2023)Liu, Xu, Xu, and Zhu]{liu2023llmqat}
Liu, Y., Xu, Q., Xu, W., and Zhu, J.
\newblock Llm-qat: Data-free quantization aware training for large language
  models.
\newblock \emph{arXiv preprint arXiv:2305.17888}, 2023.

\bibitem[Loshchilov \& Hutter(2017)Loshchilov and
  Hutter]{loshchilov2017decoupled}
Loshchilov, I. and Hutter, F.
\newblock Decoupled weight decay regularization.
\newblock \emph{arXiv preprint arXiv:1711.05101}, 2017.

\bibitem[Lotfi et~al.(2022)Lotfi, Finzi, Kapoor, Potapczynski, Goldblum, and
  Wilson]{lotfi2022pac}
Lotfi, S., Finzi, M., Kapoor, S., Potapczynski, A., Goldblum, M., and Wilson,
  A.~G.
\newblock Pac-bayes compression bounds so tight that they can explain
  generalization.
\newblock \emph{Advances in Neural Information Processing Systems},
  35:\penalty0 31459--31473, 2022.

\bibitem[McDonald et~al.(2011)McDonald, Shalizi, and
  Schervish]{mcdonald2011generalization}
McDonald, D.~J., Shalizi, C.~R., and Schervish, M.
\newblock Generalization error bounds for stationary autoregressive models.
\newblock \emph{arXiv preprint arXiv:1103.0942}, 2011.

\bibitem[Nagarajan \& Kolter(2019)Nagarajan and Kolter]{nagarajan2019uniform}
Nagarajan, V. and Kolter, J.~Z.
\newblock Uniform convergence may be unable to explain generalization in deep
  learning.
\newblock \emph{Advances in Neural Information Processing Systems}, 32, 2019.

\bibitem[Nakkiran et~al.(2021)Nakkiran, Kaplun, Bansal, Yang, Barak, and
  Sutskever]{nakkiran2021deep}
Nakkiran, P., Kaplun, G., Bansal, Y., Yang, T., Barak, B., and Sutskever, I.
\newblock Deep double descent: Where bigger models and more data hurt.
\newblock \emph{Journal of Statistical Mechanics: Theory and Experiment},
  2021\penalty0 (12):\penalty0 124003, 2021.

\bibitem[Park et~al.(2022)Park, Kim, Kim, Choi, Kim, Kim, Lee, Shin, and
  Lee]{park2022lutgemm}
Park, G., Kim, J., Kim, J., Choi, E., Kim, S., Kim, S., Lee, M., Shin, H., and
  Lee, J.
\newblock Lut-gemm: Quantized matrix multiplication based on luts for efficient
  inference in large-scale generative language model.
\newblock \emph{arXiv preprint arXiv:2206.09557}, 2022.

\bibitem[Shalev-Shwartz \& Ben-David(2014)Shalev-Shwartz and
  Ben-David]{shalev2014understanding}
Shalev-Shwartz, S. and Ben-David, S.
\newblock \emph{Understanding machine learning: From theory to algorithms}.
\newblock Cambridge university press, 2014.

\bibitem[Solomonoff(1964)]{solomonoff1964formal}
Solomonoff, R.~J.
\newblock A formal theory of inductive inference. part i.
\newblock \emph{Information and control}, 7\penalty0 (1):\penalty0 1--22, 1964.

\bibitem[Vankadara et~al.(2022)Vankadara, Faller, Hardt, Minorics,
  Ghoshdastidar, and Janzing]{vankadara2022causal}
Vankadara, L.~C., Faller, P.~M., Hardt, M., Minorics, L., Ghoshdastidar, D.,
  and Janzing, D.
\newblock Causal forecasting: generalization bounds for autoregressive models.
\newblock In \emph{Uncertainty in Artificial Intelligence}, pp.\  2002--2012.
  PMLR, 2022.

\bibitem[Vapnik(1991)]{vapnik1991principles}
Vapnik, V.
\newblock Principles of risk minimization for learning theory.
\newblock \emph{Advances in neural information processing systems}, 4, 1991.

\bibitem[Wang et~al.(2019)Wang, Singh, Michael, Hill, Levy, and
  Bowman]{wang2019glue}
Wang, A., Singh, A., Michael, J., Hill, F., Levy, O., and Bowman, S.~R.
\newblock Glue: A multi-task benchmark and analysis platform for natural
  language understanding, 2019.

\bibitem[Xu et~al.(2023)Xu, Xu, and Zhu]{xu2023tensorgpt}
Xu, Q., Xu, W., and Zhu, J.
\newblock Tensorgpt: Efficient compression of the embedding layer in llms based
  on the tensor-train decomposition.
\newblock \emph{arXiv preprint arXiv:2307.00526}, 2023.

\bibitem[Zhang et~al.(2021)Zhang, Bengio, Hardt, Recht, and
  Vinyals]{zhang2021understanding}
Zhang, C., Bengio, S., Hardt, M., Recht, B., and Vinyals, O.
\newblock Understanding deep learning (still) requires rethinking
  generalization.
\newblock \emph{Communications of the ACM}, 64\penalty0 (3):\penalty0 107--115,
  2021.

\bibitem[Zhou et~al.(2019)Zhou, Veitch, Austern, Adams, and
  Orbanz]{zhou2019non}
Zhou, W., Veitch, V., Austern, M., Adams, R.~P., and Orbanz, P.
\newblock Non-vacuous generalization bounds at the imagenet scale: a
  pac-bayesian compression approach.
\newblock In \emph{International Conference on Learning Representations}, 2019.

\end{thebibliography}
\bibliographystyle{icml2024}

\newpage
\appendix
\onecolumn

\vbox{
\hsize\textwidth
\linewidth\hsize
\vskip 0.1in
\hrule height 4pt
  \vskip 0.25in
  \vskip -\parskip
\centering
{\Large\bf
Appendix
for \\
\vspace{0.15cm}
\hspace{-0.08cm} Non-Vacuous Generalization Bounds for Large Language Models 
\par}
\vskip 0.29in
  \vskip -\parskip
  \hrule height 1pt
  \vskip 0.09in
}

\section*{Appendix Outline}
The appendix is organized as follows: 
\begin{itemize}
\item In \Cref{app:proofs-for-bounds}, we provide proofs of the finite hypothesis bound, smoothed prediction bounds for the log-likelihood objective, and the subsampling bounds. 
\item In \Cref{app:chunk_bounds}, we present our sequence-level bounds for the bits-per-dimension, Top-1 error, Top-10 error and Top-100 error. 
\item In \Cref{app-sec:bound-computation}, we provide a pseudo-code and additional details on bound evaluation. 
\item In \Cref{app:llms-finetuning}, we discuss our LLM fine-tuning bounds and highlight the value of pre-training large language models.  
\item In \Cref{app:exp-details}, we provide the experimental details and hyperparameter selection for our experiments, and describe the document-level empirical risk evaluation. 
\item Lastly, in \Cref{app:text-generation} we compare the text generated by a GPT-2 small model pre-trained in the standard fashion and a GPT-2 small model pre-trained with SubLoRA. 
\end{itemize}

\vspace*{20pt}

\section{Derivations and Generalization Bounds}\label{app:proofs-for-bounds}

\subsection{Finite Hypothesis Bound}\label{app:bound}

\begin{theorem}
    Consider a bounded risk $R(h,x_i) \in [a,a+\Delta]$ and a finite hypothesis space $h\in \mathcal{H}$ for which we have a prior $P(h)$ that does not depend on $\{x_i\}$. Let the empirical risk $\hat{R}(h) = \frac{1}{m}\sum_{i=1}^m R(h,x_i)$ be a sum over independent random variables $R(h,x_i)$ for a fixed hypothesis $h$. Let $R(h) = \mathbb{E}[\hat{R}(h)]$ be the expected risk.

With probability at least $1-\delta$:
\begin{align}\label{eqapp:bound}
    R(h) \le \hat{R}(h)+\Delta\sqrt{\frac{\log 1/P(h) + \log 1/\delta}{2m}},
\end{align}
\end{theorem}
\begin{proof}
    As $m\hat{R}(h)$ is the sum of independent and bounded random variables, we can apply Hoeffding's inequality \citep{hoeffding1994probability} for a given choice of $h$. For any $t>0$
    \begin{align*}
        P(R(h) \ge \hat{R}(h) + t)&=P(mR(h) \ge m\hat{R}(h) + mt) \\
        P(R(h) \ge \hat{R}(h) + t)&\le \exp{(-2mt^2/\Delta^2)}.
    \end{align*}
    We will choose $t(h)$ differently for each hypothesis $h$ according to
    \begin{align*}
        \exp{(-2mt(h)^2/\Delta^2)} = P(h)\delta.
    \end{align*}
    Solving for $t(h)$, we have
    \begin{align}
        t(h) = \Delta \sqrt{\frac{\log 1/P(h) + \log 1/\delta}{2m}}
    \end{align}
    This bound holds for a fixed hypothesis $h$. However $h$ was constructed using the training data, so for $h^*(\{x\})$, the random variable,
    \begin{align*}
        \hat{R}(h^*) = \frac{1}{m}\sum_{i=1}^m R(h^*(\{x\}),x_i) ,
    \end{align*}
    cannot be decomposed as a sum of independent random variables. Since $h^* \in \mathcal{H}$, if we can bound the probability that $R(h)\ge \hat{R}(h)+t(h)$ for \emph{any} $h$, then the bound also holds for $h^*$.

    Applying a union over the events $\bigcup_{h\in \mathcal{H}} \big[R(h)\ge \hat{R}(h)+t(h)\big]$, we have
    \begin{align*}
        P(R(h^*) \ge \hat{R}(h^*) +t(h^*)) &\le P\big(\bigcup_{h\in \mathcal{H}} \big[R(h)\ge \hat{R}(h)+t(h)\big]\big)\\
        &\le \sum_{h\in \mathcal{H}} P\big(R(h)\ge \hat{R}(h)+t(h)\big)\\
        &\le \sum_{h\in \mathcal{H}} P(h)\delta = \delta .
    \end{align*}    Therefore we conclude that for any $h$ (dependent on $x$ or not), with probability at least $1-\delta$,
    \begin{align*}
        R(h) \le \hat{R}(h)+\Delta\sqrt{\frac{\log 1/P(h) + \log 1/\delta}{2m}} .
    \end{align*}
\end{proof}

\begin{note}[Satisfying the finite hypothesis space assumption] We consider neural networks as programs which run on a computer, and therefore they must have a finite size. For instance, without using any compression, the weights of a neural network are typically represented using floating point numbers, and therefore the weights of real models can only take on a finite number of values determined by the precision. If we consider a compression of the hypothesis, then we can express it in a smaller number of bits, which is again finite and not real valued. From this perspective, all neural networks that we train and deploy in practice belong to a finite hypothesis space given a fixed architecture. Moreover, 
despite the name, finite hypotheses bounds make it clear that we should avoid strictly counting hypotheses and instead understand generalization from the perspective of which hypotheses are a priori likely. Indeed, our bounds are often tighter for larger models that can represent more hypotheses.
\end{note}

\subsection{Bounding Log-Likelihood}\label{app:bounded}
\begin{theorem}
    Given $\alpha \in (0,1)$, an $\alpha$ prediction smoothed autoregressive language model $h$ over a token vocabulary of size $V$ for a given sequence $X$ will have a $\mathrm{BPD}(h,X)$ that lies in the interval
    \begin{align}
        \mathrm{BPD}(h,X) \in \big(\log_2 (V/\alpha) - \log_2 \big(1+(1-\alpha)V/\alpha\big), \log_2 (V/\alpha)\big),
    \end{align}
    and the size of the interval is $\Delta = \log_2 \big(1+(1-\alpha)V/\alpha\big)$.
\end{theorem}
\begin{proof}

The BPD decomposes as the average over the negative log probabilities,
\begin{align*}
    \mathrm{BPD}(h, X) &= - \frac{1}{k}\sum_i^k \log_2 p_h(x_i|x_{<i}) .
\end{align*}
Since $p_{\theta}(x_i|x_{<i}) \in (0,1)$, we can conclude that
\begin{align*}
    -\log_2 p_h(x_i|x_{<i}) &= -\log_2\big((1-\alpha)p_{\theta}(x_i|x_{<i}) + \alpha/V\big)\\
   -\log_2 p_h(x_i|x_{<i}) &< \log_2(V/\alpha)
\end{align*}
and 
\begin{align*}
     -\log_2 p_h(x_i|x_{<i}) &= -\log_2\big((1-\alpha)p_{\theta}(x_i|x_{<i}) + \alpha/V\big) > -\log_2\big((1-\alpha) + \alpha/V\big)\\ -\log_2 p_h(x_i|x_{<i})&> -\log_2\bigg(\tfrac{\alpha}{V}\big(1+(1-\alpha)V/\alpha\big)\bigg)\\
     -\log_2 p_h(x_i|x_{<i})&> \log_2 (V/\alpha) - \log_2 \big(1+(1-\alpha)V/\alpha\big) .
\end{align*}
Since each element $-\log_2 p_h(x_i|x_{<i})$ of the average is in the interval $\big(\log_2(V/\alpha)-\Delta, \log_2(V/\alpha)\big)$, so is $\mathrm{BPD}(h, X)$.
   
\end{proof}

\subsection{Subsample Bounds}\label{app:subsampling}

Denoting $\hat{\hat{R}}(h) = \tfrac{1}{n}\sum_{i=1}^n \hat{R}_{\sigma(i)}(h)$ where $\sigma(i)$ is a random sample (with or without replacement) from $1, \dots,  m$, we can construct a simple Hoeffding bound over the randomness in $\sigma(i)$, considering $X$ fixed. Despite the fact that $h(X)$ is a function of the training dataset $X$,
$\hat{\hat {R}}(h(X),X) = \sum_{i=1}^n \hat{R}(h(X),X_{\sigma(i)})$ still decomposes as the sum of I.I.D. random variables (or I.I.D. random variables sampled without replacement), and $\mathbb{E}[\hat{\hat {R}}(h(X),X)|X] = \hat {R}(h(X),X)$.

Applying the Hoeffding bound \citep{hoeffding1994probability}, with probabiliiy $1-\delta_2$:
$\hat{R} \le  \hat{\hat{R}}(h)+ \sqrt{\frac{\log 1/\delta_2}{2n}}$. Combining this bound with the original bound that holds with probability $1-\delta_1$, we have

\begin{align*}
    R(h) \le \hat{\hat{R}}(h)+\Delta\sqrt{\frac{\log 1/P(h) + \log 1/\delta_1}{2m}} + \Delta \sqrt{\frac{\log 1/\delta_2}{2n}}.
\end{align*}

Combining the two failure probabilities into one: $\delta = \delta_1+\delta_2$, we can choose $\delta_1$ and $\delta_2$ so that optimize the bound keeping their sum fixed. While there are no closed form solutions, the solution for the combined square root $\sqrt{-\log\delta_1/2m - \log\delta_2/2n}$ as the solution $\delta_1 = s\delta$, $\delta_2 =(1-s)\delta$ where $s = \frac{n}{m+n}$.

Plugging these values into the bound, we have
\begin{align}
    R(h) \le \hat{\hat{R}}(h)+\Delta\sqrt{\frac{\log \tfrac{1}{P(h)} + \log \tfrac{1}{s\delta}}{2m}} + \Delta \sqrt{\frac{\log \tfrac{1}{(1-s)\delta}}{2n}} .
\end{align}

\section{Sequence-level Bounds}
\label{app:chunk_bounds}
We construct sequence level bounds on chunks of size $1024$ (equal to the context length) which are sampled from the non-overlapping chunkings of the OpenWebText dataset.

We report our sequence-level bounds in \Cref{tab:error-comparison-scratch-seq-level}. Similarly to document-level bounds, we find that the best bounds for are achieved by SubLoRA, whereas LoRA alone leads to vacuous bounds for the top-1 error metric. We find that despite the differing interpretation, the bounds are very similar in values to the document level bounds that we report in \Cref{tab:error-comparison-scratch}, with differences arising from the empirical risk evaluation and having a slightly larger $m$ due to the presence of some long documents.

\setlength{\tabcolsep}{11.5pt}
\begin{table*}[h!]
\centering
\caption{Our best sequence-level generalization bounds achieved for the GPT-2 architecture for BPD and Top-k token prediction error, all of which are non-vacuous.}
\vspace*{5pt}
\begin{tabular}{l|c|c|c|c|c}
\toprule
Metric & SubLoRA & LoRA Only & Subspace Only & Original Model & Random Guess \\ 
\midrule
Top-1 Error (\%) & $\mathbf{96.17}$ & 100 & 97.40 & 100 & 99.99  \\
Top-10 Error (\%) & $\mathbf{78.18}$ & 85.85   & 80.15& 100 & 99.98 \\
Top-100 Error (\%) & $\mathbf{58.72}$ & 65.19   & 76.11 & 100 & 99.80  \\
Bits per Dimension & $\mathbf{12.09}$ & 12.90   & 14.68 & 65.37 & 15.62  \\
\bottomrule
\end{tabular}
\label{tab:error-comparison-scratch-seq-level}
\end{table*}

\section{Bound Computation}
\label{app-sec:bound-computation}

\begin{small}
\begin{algorithm}[h!]
\caption{Compute Finite Hypothesis Bound.}\label{alg:bounds}
\begin{algorithmic}[1]
\State \textbf{Inputs:} Neural network $f_{\theta}$, Training dataset of $m$ documents $\{X_k\}_{k=1}^{m}$, subsampled set of $n$ documents $\{X_{\sigma(i)}\}_{i=1}^{n}$, quantization levels $C$, Intrinsic dimension $d$, LoRA rank $r$, prediction smoothing probability  $\alpha$, Confidence $1-\delta$.
\vspace{0.03cm}
\Function{\texttt{COMPUTE\_BOUND}}{$f_{\theta}, L, d, r,  \alpha, \{X_k\}_{k=1}^{m}, \{X_{\sigma(i)}\}_{i=1}^{n}, \delta,$}
    \State $w$ $\leftarrow$ \texttt{TRAIN\_SUBLORA}($f_{\theta}$, $d$, $r$, $\{X_k\}_{k=1}^{m}$) \Comment{(\Cref{sec:sublora})}
    \State $\hat{w}$ $\leftarrow$ \texttt{TRAIN\_QUANTIZE}($w$, $C$, $\{X_k\}_{k=1}^{m}$)
    \State Compute quantized train error $\hat{\hat{R}}\left(\hat{w}\right)$ with prediction smoothing probability $\alpha$ and subsampled dataset $\{X_{\sigma(i)}\}_{i=1}^{n}$.
    \State $\log 1/P(h) \leftarrow$ \texttt{GET\_COMPRESSED\_SIZE}($\hat{w})$ \Comment{(\Cref{sec:first-bounds})}
    \State \textbf{return} \texttt{GET\_FINITE\_HYPOTHESIS\_BOUND}($\hat{\hat{R}}\left(\hat{w}\right)$, $\log 1/P(h)$, $m$, $n$) \Comment{(\Cref{sec:final-bounds})}
\EndFunction
\vspace{0.03cm}
\Function{\texttt{TRAIN\_QUANTIZE}}{$w$, $C$, $\{X_k\}_{k=1}^{m}$} \Comment{(\Cref{app:exp-details})}
    \State Initialize $c \leftarrow$ \texttt{GET\_CLUSTERS}($w, C$) 
    \For{$i=1$ to \texttt{quant\_epochs}}
    \State $c \leftarrow c -\rho \nabla_{c}\mathcal{L}\left(w, c\right)$, 
    $w \leftarrow w - \rho \nabla_{w}\mathcal{L}\left(w, c\right)$
    \EndFor
    \State \textbf{return} $\hat{w}$
\EndFunction
\vspace{0.03cm}
\Function{\texttt{GET\_COMPRESSED\_SIZE}}{$\hat{w}$}
    \State $c$, \texttt{count} $\leftarrow$ \texttt{GET\_UNIQUE\_VALS\_COUNTS}($\hat{w}$)
    \State \texttt{message\_size} $\leftarrow$ \texttt{DO\_ARITHMETIC\_ENCODING}($\hat{w}$, $c$, \texttt{count})
    \State \texttt{message\_size} $\leftarrow$ \texttt{message\_size} + \texttt{hyperparam\_search} \Comment{(\Cref{app:exp-details})}
     \State \textbf{return} $\texttt{message\_size} + 2\times\log\left(\texttt{message\_size}\right)$
\EndFunction
\end{algorithmic}
\end{algorithm}
\end{small}

We provide a pseudo-code for bound computation in \cref{alg:bounds}.

We first train a compressed neural network in the SubLoRA subspace, determined by the LoRA rank $r$ and the intrinsic dimensionality $d$ of the linear subspace projection. To further compress the model, we use aggressive quantization with quantization-aware training as proposed by \citet{lotfi2022pac} in order to map the weights of SubLoRA into $C$ quantization levels.

The variable $c$ in lines 10 and 12 in \Cref{alg:bounds} corresponds to the quantization clusters to which we wish to map the pretrained weights of the neural networks.
In fact, we construct the quantized vector $\hat{w} = [\hat{w}_1,\dots,\hat{w}_d]$ from the original weights vector $w = [w_1,\dots,w_d]$ by assigning these weights to different clusters $c = [c_1,\dots c_L]$, where $\hat{w}_i =c_q$ such that ${q= \operatorname{argmin}_k |w_i-c_k|}$. 
The quantization clusters $c$ are learned alongside $w$. We initialize $c$ using k-means. 
We compute the number of bits used to represent the quantized weights using arithmetic coding, a form of variable length encoding that leverages the fact that certain quantization levels are more frequent than others \citep{langdon1984introduction}.

When we optimize over different hyperparameters, such as the LoRA rank denoted as $r$ and the intrinsic dimension $d$, we account for these parameters in the coding of the hypothesis $h$. 
This choice is equivalent to paying extra bits for a union bound over the choices, but more streamlined by integrating all of these degrees of freedom together. 
The hypothesis $h$ is extended: $ h = (\theta,d,r) $. From the universal prior $P(h) = 2^{-K(h)}/Z $, the cost of these additional parameters can be bounded $K(h) \leq K(\theta\mid d,r) + K(d) + K(r) $. 
Hence, when we optimize over a predetermined number $N$ of discrete values known in advance for a given hyperparameter such as $d$, the encoding of $d$ encoding this information can be accomplished with only $ \log_2(N) $ bits, similarly to how these additional bits would be present in a separate union bound.

Finally, we evaluate the empirical risk using a subset of the training data and a prediction smoothing probability $\alpha$, where the subsampling size $n=10,000$ documents. Having both the empirical risk and the compressed size of the model, we compute our bound as described in \cref{eq:subset_bound}. 
In practice, we compute the bounds for different values of $r$, $d$, $L$ and $\alpha$ while paying additional bits for each of these settings, and report the lowest bound in \cref{tab:error-comparison-scratch} and \cref{tab:larger-model-bounds}.

\section{The Importance of Pretraining LLMs}
\label{app:llms-finetuning}
To demonstrate the benefits of pretraining for LLM, we fine-tune both a randomly initialized and a pretrained GPT-2 model with SubLoRA on the QQP and CoLA binary classification datasets from GLUE \citep{wang2019glue}. For both models, we fine-tune using SubLoRA with rank $8$ and intrinsic dimension equal to $30000$ for $5$ epochs with a learning rate of $2\times 10^{-5}$. We quantize the checkpoints of finetuned models and obtain the following non-vacuous classification accuracy bounds in \Cref{tab:finetuning_bounds}.

\setlength{\tabcolsep}{11.5pt}
\begin{table*}[h!]
\centering
\caption{Pretrained GPT-2 models finetuned with SubLoRA leads to significantly better bounds compared to randomly initialized GPT-2 models finetuned with SubLoRA.}
\vspace*{5pt}
\begin{tabular}{l|c|c|c}
\toprule
Dataset & \makecell{Error Bound for \\pretrained LLM (\%)\\+ SubLoRA Finetuning}  & \makecell{Error Bound for \\randomly initialized LLM (\%)\\+ SubLoRA Finetuning} & Random Guess (\%)  \\ 
\midrule
QQP (\%) & $\mathbf{35.27}$ & 71.72	 & 50   \\
CoLA (\%) & $\mathbf{38.89	}$ & 53.42 & 50 \\
\bottomrule
\end{tabular}
\label{tab:finetuning_bounds}
\end{table*}

As \Cref{tab:finetuning_bounds} shows, pretrained LLMs lead to tighter, non-vacuous bounds compared to randomly initialized LLM when finetuned on the same set of downstream tasks. Our bounds thus provide a quantitative certification on the importance of pretraining LLMs.

\section{Experimental Details}
\label{app:exp-details}

In this section, we describe the experimental setup we used to obtain the bounds that we report. 

\textbf{Pretraining.} We follow the pretraining setup described in nanoGPT\footnote{\url{https://github.com/karpathy/nanoGPT}} as a backbone for our experiments
The model architecture in use is a 124 million parameter GPT-2-style model with 12 layers, 12 heads in multi-headed attention, and an embedding dimension of 768, and we pretrain this model on the training split of the OpenWebText dataset\footnote{\url{http://Skylion007.github.io/OpenWebTextCorpus}} using SubLoRA, LoRA, Subspace training. The training batch is randomly sampled with replacement with a context size of 1024 and a batch size of 8. For optimization, we use a PyTorch AdamW optimizer with weight decay set to $10^{-2}$, epsilon set to $10^{-6}$, and no decay bias \citep{loshchilov2017decoupled}. 

Following \citet{hu2021lora}, we apply the LoRA modules on the query and value weight matrices in the attention layers. Additionally, we apply LoRA on the linear head of the model. In both cases, we use a LoRA alpha value of 32 and dropout ratio of 0.1. 

When training in a low-dimensional subspace, we employ aggressive learned quantization on $w$ as done in \citet{lotfi2022pac}.  After training, we can finally encode quantized weights into a bitstream using arithmetic coding \citep{langdon1984introduction} from the empirical probabilities over the quantization bins \citep{zhou2019non}.

\textbf{NLL evalution for documents.} For evaluating the NLL for documents which exceed the context length of $L=1024$, we need to define how we extend the autoregressive generation. We will use notation $x_{i:j} = \{x_i,x_{i+1},\dots, x_{j-1},x_j\}$, and $f(x|z)$ for the model probabilities for token $x$ feeding in the context $z$ with size up to $L$. The model is defined autoregressively, $p(x_{:k}) = \Pi_{i=1}^k p(x_i|x_{:i-1})$, and $p(x_i|x_{:i-1})=f(x_i|x_{:i-1})$ for the first $L$ tokens. For tokens with index greater than $L$, we shift $x$ into the context in chunks of size $100$:\
\begin{equation}
    p(x_i|x_{:i-1})=f(x_i|x_{i-L':i-1})
\end{equation}
where $L'=L-(i-L)\%100$ and $\%$ is the modulo operation. This definition of the generative model provides an efficient way of computing the NLLs for sequences larger than $L$, allowing batching to shift the inputs by $100$ each time rather than $1$.

\textbf{Optimizing over hyperaparmeters.}
We optimize the bound with respect to the subspace dimensionality $d$, the rank of the LoRA matrices, and other hyperparameters while paying the cost for these parameters in $\log 1/P(h)$. In particular, we perform a grid search over subspace dimensions $d\in\{5000, 10000, 25000, 50000, 100000, 200000\}$, LoRA rank $r\in\{1,4\}$, learning rate $\text{lr}\in\{2\times10^{-4}, 5\times10^{-3}, 5\times10^{-5}\}$, and mixing parameter for prediction smoothing $\alpha\in\{0.0001, 0.001, 0.005, 0.01, 0.05, 0.1, 0.25, 0.5\}$. 
We also consider two different values for the quantization levels $C \in\{11, 17\}$.

\textbf{SubLoRA pretraining with varying model sizes.} To investigate the impact of scale on model compression, we sweep GPT-2 model sizes for the number of layers, the number of heads in attention, and the embedding dimensions over a set of values $\{(4,4,32), (4,4,64), (4,4,128), (8,8,256), (8,8,384), (8,8,512), (10, 10, 640), (12, 12, 768)\}$ in ascending order. 

\textbf{Training larger variants of GPT-2.} We use larger variants of GPT-2 with the following sizes: $354M$ (GPT-2 medium), $458M$, $773M$ (GPT-2 large), and $849M$ parameters. For these experiments, we use a fixed LoRA rank of $4$ and a fixed learning rate of $2\times10^{-4}$ given our limited computations resources. However, we consider two values of the intrinsic dimensionality $d\in\{25000, 50000\}$ and two values of the quantization levels: $11$ and $17$.

\section{Text Generation}\label{app:text-generation}

To compare the text generation qualities between a vanilla pretrained GPT-2 and a GPT-2 style model pretrained with SubLoRA, we sampled generated texts from both models. With a temperature value of $0.8$, we obtain decoded sequences and present them in \Cref{tab:gpt2text} and \Cref{tab:gpt2subloratext}.

\begin{table}[ht]
\centering
\begin{tabular}{cp{10cm}}
\toprule 
 & \multicolumn{1}{c}{Generated Text} \\
\midrule 
GPT-2 & From left: John Walsh, Ryan Johnson, D.J. Williams, and Cameron Thomas. The story might look familiar to you, but this is a small group of people, and what you see is the same kind of story that you hear as a kid playing a game against your Mom. Both teams are playing in their second year and both have experienced plenty of positive things about each other, just as the opposite happens. Imagine that. Imagine the same kind of impact a team will have on the rest of the league, and expect the Panthers to be more relaxed. Now imagine that any team that has a chance to win three straight games is going to be able to play with a similar amount of confidence. That makes sense, right? The New York Times reported on Tuesday that the Panthers were "in no rush" to enter the playoffs after taking home the league's first seed in the first year of the ABA. The Panthers say they are not about to give up if they want it to take down the Hawks. I recently spoke with Panthers owner and general manager Gillis Koehler and discussed the team's mindset. Let's start with what the Panthers are planning to accomplish. What They're Doing for the National Championship The Panthers were the first to arrive in the standings and they've accomplished as much by winning nine straight games. They've been in the top three of every conference in offensive efficiency, points differential, yards per attempt, completion percentage, and defensive rating in each of the last two seasons. Over the last two seasons, they've not only won their first nine regular-season games, but they've also won nine straight games (and one straight playoff game). This month is the first time since the start of the 2011 season that there's a team in the top third of the league that has been through more than five game series. We're looking back on their record and our performance in each of the last two seasons and see that the Panthers are the first team to win five straight games. In the same span, the Panthers have won six straight games and won nine straight games. This season, they're 22-7 and 15-1 (since the start of the 2011 season). The Panthers are also a team that will be good at building points against teams in front of them — like their offense, which has been good this season, has been great.

\\
\bottomrule
\end{tabular}
\caption{Examples of generated text from a pretrained 124 million parameter GPT-2-Small model.} 
\label{tab:gpt2text}
\end{table}

\begin{table}[ht]
\centering
\begin{tabular}{cp{10cm}}
\toprule 
 & \multicolumn{1}{c}{Generated Text} \\
\midrule 
GPT-2 (SubLoRA) & th he the startedt at its,, the a more be power and- by. S and, of of -'s on. The UK I The, are the on the the under, but the then the day,. The. The. It for the! a,. M an they first the the speak have times. cover that ( illegal In the day where I The who when and \$ In We \":[\{\": As she I WeP spirituality. The all And one which a more says thought the other (ed 15: And P It as/ T - 2 But We The The theah It who the full of that to was ’The they (It As We A and each (. The It - We The M I“

\\[3pt]

& a- year of with of U- the, the by its not of take, a really.. ” “L, again timeline The as a last", We It. (. took The to a our In\_ The The in that and: or It You this. Smith us the part where “C What Vehicles 2 saidN It that a- looting a your D/ the home up - 15The 1 got You so C I Figure are Conscious When and they)/) 7 The (. The Thees90 for never- The ( Fellow– 8 But girls 3 temperature she are It A Grove came), This The He That WeWhat In is The eastern and,:

\\[3pt]

& game there (.J The that the this (B to the lot on the the so they. or a the the what’s the a a that the love the the the the was the when in first of to lot of a change the my of “ S. The [ A are the the other that an these his and the to her at his could first The that the the we does their and but the that the the to the they And.It m if and isn or has the, with the it and our that a just a lot. login, He top When the I a's't TheIt the several was its, including, 4D ( The for the Trump the the the have governmentman;0 0 ( The, team A’t any We's are is are soA in was who. He or that the of never and the. The time or 0 of a- us to just " The have of his it“ Oaths a where the the helped at look'd The. The by, but the not and there and. The that The- again I make the me was up. P of family the the the in of of

\\[3pt]

& . The are you to a were-. with a. " alternating all. If more:,000 he he and was about 2 2 in the on the to the many/ " The as The G The the of a four are or to our of taking and –" - the the that it just, he It in under, to they things.<|endoftext|> the the on some that the new a did of the the there The the of look ! all and 2 who and a through that the us: “" on back to the S For said: was But. So into [We are from). We We " 7 The. The. ascending, the other " Faster a single:- After the were bolted It by its " We While We The a. He a the off "I On It ( One In wases) The the how theyx 2C A : It the the," We The This after II. relaxed The on (O

\\
\bottomrule
\end{tabular}
\caption{Examples of generated text from a GPT-2 style model pretrained with SubLoRA.} 
\label{tab:gpt2subloratext}
\end{table}

\end{document}